\documentclass{article}
\usepackage{ijcai18}
\usepackage{amsmath,amsthm,amssymb,mathtools}
\usepackage{hyperref}
\usepackage[active]{srcltx}
\usepackage{xcolor,helvet,courier,esint,tabularx,comment}
\hypersetup{colorlinks=true,citecolor=green}
\usepackage{algorithm,algpseudocode}
\usepackage{accents}
\usepackage{relsize}
\usepackage{times}
\usepackage{xcolor}
\usepackage{soul}
\usepackage[utf8]{inputenc}
\usepackage[small]{caption}
\usepackage{titlesec}
\titlespacing\section{0pt}{4pt plus 4pt minus 2pt}{4pt plus 2pt minus 2pt}
\titlespacing\subsection{0pt}{4pt plus 4pt minus 2pt}{4pt plus 2pt minus 2pt}
\titlespacing\subsubsection{0pt}{4pt plus 4pt minus 2pt}{4pt plus 2pt minus 2pt}

\newtheorem{theorem}{Theorem}
\newtheorem{definition}[theorem]{Definition}
\newtheorem{lemma}[theorem]{Lemma}

\newtheorem{corollary}[theorem]{Corollary}
\newtheorem{proposition}[theorem]{Proposition}

\newcounter{fncntr}

\algdef{SE}[DOWHILE]{Do}{doWhile}{\algorithmicdo}[1]{\algorithmicwhile\ #1}%

\newcommand\algoLUU[0]{\textproc{Learn-Utility} }

\title{An Online Algorithm for Learning Buyer Behavior under Realistic Pricing Restrictions%\vspace*{-5cm}
}
\author{
Debjyoti Saharoy $^1$ \normalfont{and} 
\textbf{Theja Tulabandhula} $^2$,   
\\ 
University of Illinois at Chicago.\\
\texttt{ \{$^1$dsahar2, $^2$theja\}@uic.edu.}
}

\begin{document}
\setlength{\abovedisplayskip}{0pt}
\setlength{\belowdisplayskip}{0pt}
\setlength{\abovedisplayshortskip}{0pt}
\setlength{\belowdisplayshortskip}{0pt}

\maketitle

\begin{abstract}
  We propose a new efficient online algorithm to learn the parameters governing the purchasing behavior of a utility maximizing buyer, who responds to prices, in a repeated interaction setting. The key feature of our algorithm is that it can learn even non-linear buyer utility while working with arbitrary price constraints that the seller may impose. This overcomes a major shortcoming of previous approaches, which use unrealistic prices to learn these parameters making them unsuitable in practice. 
\end{abstract}

\section{Introduction}
\label{sec:introduction}

Modeling the arrival and response behavior of a buyer to a collection of items sold by a seller has a rich history in operations management\cite{cohen2016feature} and machine learning~\cite{kleinberg2003value,amin2014repeated}, and helps answer questions such as: what assortment of items should a seller show a prospective buyer? How should she price them? Much work in this area can be divided into two categories: (1) explicitly learning the purchase model, and (2) maximizing the revenue or some other function given a behavior model. 

In particular, online problems in the latter category ~\cite{besbes2015surprising,chakraborty2009dynamic,alaei2011bayesian,cai2011extreme,blum2011welfare}, instead of learning the buyer behavior, optimizes what is known as \emph{regret}, which is the difference between what the seller could have done in hindsight compared to what they did in a sequence of interactions with the buyer. Although the regret setting is appealing, the techniques and the corresponding algorithms tend to be very specialized (except for perhaps the simplest cases) and lacks universality. In particular, many of the general purpose algorithms (such as Thompson Sampling and UCB) depend linearly on the number of actions, which is not-ideal when the action space is large or infinite (as is the case for us). Specialized analysis or algorithms address this dependence issue but depend heavily on the structure of the objective and the decision problem. Further, if the objective or the decision structure changes, either because of business considerations or as new business logic is introduced, one has to design new algorithms from scratch. Thus, it is economical and convenient to decouple the estimation problem from the decision making problem and explicitly estimate the parameters of buyer behavior first  
(also called \emph{pure exploration}).

%benefit of realism
There is a recent line of work on \textit{learning the behavior of buyers} online~\cite{Balcan2014,beigman2006learning,bei2016learning}. Compare to these works, our algorithm does not share a key shortcoming, which is the necessity of posting unrealistic prices in the process of learning. Note that, learning buyer behavior in the offline (batch data) setting has also been addressed in recent works. For instance, in \cite{benkdd16},  the authors learn the parameters of a particular buyer behavior model that considers preference lists. We believe the online setting is relatively more interesting because there is scope for real-time personalization tailored to each individual buyer compared to  the offline setting. 

%buyer models 
In this paper we consider a buyer behavior where the buyer's objective is sensitive to prices. This type of sensitivity to prices to prices has been considered in the online setting~\cite{STOC16p949} in the context of regret minimizing profit maximization, as well as in the offline setting~\cite{benkdd16}. In each interaction, the seller prices a collection of items and the buyer responds by purchasing various quantities of each item that maximizes her objective/utility. 
Previously proposed algorithms resort to posting unrealistic prices to induce the buyer to buy/not-buy certain items. Our algorithm eliminates this shortcoming by learning the buyer behavior while being constrained to post prices from a predefined set of realistic prices (described in Section \ref{sec:realistic-prices}), which is provided as an input. 
The practical motivation for such a constraint to be imposed is straightforward: prices of items in many commercial settings are only allowed to vary between  realistic lower and upper bounds. 
This is because of business constraints and prior knowledge on the market value of goods and items. More involved constraints include bundle prices (where prices are tied to each other) and promotion/discount prices that are also specified by business rules. Further, all such constraints can vary arbitrarily over time. 

The fact that our algorithm can learn despite such pricing constraints makes it practical and applicable in real scenarios. Internally, it exploits the concavity property of the buyer's objective and uses projected gradient descent to shrink an uncertainty ellipsoid around the true buyer model parameters. 

\subsubsection{Our Technical Contribution}
\begin{enumerate}
\item Unlike Algorithm 2 in~\cite{Balcan2014} (in a buyer model related to the one considered here) which breaks when prices are restricted, our online learning algorithm is able to make progress. Specifically, it gets $\epsilon$-optimal estimates in roughly $O(n^2 \log \frac{n}{\epsilon})$ interactions even for non linear buyer utility (versus $\tilde{O}(n)$ taken by the former to learn linear utility). Thus, we provide a bound (Theorem \ref{thm:stackelberg_utility_learn}) on the number of interactions our algorithm needs for learning the buyer parameters. Reducing the number of interactions is important because while learning, the algorithm is agnostic to the revenue generated.

\item The key feature of our algorithm is that it searches for "realistic" prices (Lemma \ref{lemma:lemma_from_STOC16_adaptation}) that induce purchase of specific target bundles (Theorem \ref{lemma:compute g(p)}), and creates hyperplanes corresponding to these price vectors to sequentially split the uncertainty set over the buyer's parameters.
\end{enumerate}

\subsubsection{Comparison With Previous Work}
From ~\cite{STOC16p949}, we re-purpose the use of a gradient descent based technique (used in solving the convex program in Equation (\ref{eq:reformulated_buyer})) for interacting with the buyer. While they do not need any specific variant of the gradient descent algorithm, we explicitly choose a certain step rule (constant step length) to bound our learning errors. Our own contribution here is the use of these gradient descent moderated interactions in a "realistic" price space (an additional caveat) to enable the seller
(1) approximately learn the "value" of the goods purchased by the buyer without knowing his "inherent" utility (Lemma \ref{lemma:lemma_from_STOC16_adaptation}, Theorem \ref{lemma:compute g(p)}) and, 
(2) split the uncertainty ellipsoids, whereas ~\cite{STOC16p949} use such interactions for solving a specific structured Stackelberg game (requiring very different tools and techniques in the process).

Similar to ~\cite{cohen2016feature}, we use two specific eigenvalue lower bounding lemmas (see Lemma~\ref{cohen_lemma4} and ~\ref{cohen_lemma5}) to bound the number of rounds of interaction needed by our algorithm.
Their proposed algorithm, which essentially does a multidimensional binary search for the best price to post in each round can break if there is a coupling across items. On the other hand, the search procedure our algorithm follows can handle such coupling as the ellipsoidal procedure searches for parameters related to the entire universe of items.
Another key difference between our setting and ~\cite{cohen2016feature} is that we work with multiple items in each round in contrast to their single item setting. While using an ellipsoid to represent uncertainty in parameter estimation, the cut direction and the hyperplane placement is straightforward in ~\cite{cohen2016feature}. On the other hand, in our algorithm, in the presence of realistic pricing constraints we carefully choose the cut direction as well as position the separating hyperplane by solving the dual of a specific optimization problem using projected gradient descent.
Finally, note that their problem is a version of the contextual bandit problem for which general purpose algorithms are already available, whereas our problem is not a contextual bandit instance. 
Further, we note that algorithms in both ~\cite{STOC16p949} and ~\cite{cohen2016feature} cannot be easily extended to the realistic prices setting (defined in Section~\ref{sec:realistic-prices}), which is our key emphasis here.

For the buyer models that we consider, the utility $U(x)$ need not be linear in the bundle, so even polynomial utility functions can be learned as long as certain conditions mentioned in the assumptions of Section~\ref{sec:realistic} are met. This makes our algorithm and its analysis in Section~\ref{sec:realistic} more generally applicable. 

\section{Realistic Prices}
\label{sec:realistic-prices}

Here we define what we mean by realistic prices which will constraint the prices that the seller can set while learning the buyer model (Section~\ref{sec:realistic}).

%Definition
The price $p_i$ of an item $i$ is realistic if it lies within the interval $[p^0_i-\delta_i, p^0_i+\delta_i]$, where $p^0_i$ is the median price point and $2\delta_i$ is the length of the interval (without loss of generality, we can assume symmetry here). This leads to an $n$-orthotope, which is defined as follows:
\begin{definition}\label{def:realistic-price}
A set of prices is said to be realistic if it is of the following form: 
\begin{equation}
\label{eq:realistic_price2}
\mathcal{P} = \left\lbrace p \in \mathbb{R}^n_+ \;|\; \lVert S^{-1} (p-p^0)\rVert_\infty \leq 1 \right\rbrace,
\end{equation}
where $p^0 \in \mathbb{R}^n_+$ is the median price point, $\Delta = [\delta_1 \cdots \delta_n]^T \geq \mathbf{0}$, 
and $S=\text{diag}(\delta_1,\cdots,\delta_n)$
\begin{comment}
\[
  S =
  \begin{bmatrix}
    \delta_1 & & \\
    & \ddots & \\
    & & \delta_n
  \end{bmatrix}
\]
\end{comment}
is its corresponding diagonal matrix. The length of the realistic price interval for each item $i \in [n]$ is thus $2\delta_i$.
\end{definition}

\begin{comment}

Thus, we consider prices that belong to a transformed $\lVert \cdot \rVert_{\infty}$ norm-ball (scaled by matrix $S$ and translated by $p^0$). While price covariation constraints are not captured in this definition, suitable affine transformations and other convex bodies such as ellipsoids can also be be used to restrict prices. 
\end{comment}
For analysis, we will assume that the set $\mathcal{P}$ is enclosed in a Euclidean ball of radius $R$.
\begin{comment}
%practical motivation
The practical motivation for such a constraint to be imposed while learning the buyer behavior model is straightforward: prices of items in many commercial settings are only allowed to vary between  realistic lower and upper bounds (for instance, setting price as $0$ or $\infty$ is impractical). This is because of business constraints and because of prior knowledge on the market value of goods and items. More involved constraints include bundle prices (where prices are tied to each other) and promotion/discount prices that are also specified by business rules. Further, all such constraints can vary arbitrarily over time. We do not make any strong assumptions on their description and only consider them as exogenous inputs for our algorithm. 

%previous algos don't work
Restricting prices has a profound impact on algorithms that have been previously proposed in the literature. For instance, Algorithm 2 in~\cite{Balcan2014} (that addresses learning a slightly different variant of the buyer model considered here) breaks down when one restricts the prices to $\mathcal{P}$. As we show in the subsequent section, our algorithm is able to make progress and outputs an uncertainty set that contains the true parameter of the buyer model even when prices are restricted.

\end{comment}
%\onecolumn
\section{Buyer Model}
\label{sec:realistic}
%Buyer
We represent a bundle of goods $x\in C\subseteq [0,1]^n$ (where $C$ is the feasible set) by a vector representing what fraction of each of the $n$ goods is purchased. The prices are represented by a vector $p = (p_1, \cdots, p_n) \in \mathbb{R}^n$. The price of a bundle $x$ is simply $p^Tx=\sum_{i=1}^{n} p_i \cdot x_i$. When the buyer is provided with a price vector $p$, she buys the \emph{tie-breaking} utility maximizing bundle, $x^*(p)$, which is the optimal solution of the following optimization program:
\begin{equation}
\label{eq:Buyer's Problem2}
\begin{array}{ll@{}ll}
	x^*(p)\; =\; \arg\max_{x \in C}  &U(x) + \frac{4}{\mu}\left(\mathlarger \sum \limits_{i=1}^{n} \sqrt{x_i}\right) - p^Tx.&\\
\end{array}
\end{equation}

Ideally, a utility maximizing buyer would maximize $U(x) - p^Tx$, where $U:[0,1]^n \rightarrow \mathbb{R}$ specifies their utility for each possible bundle. Since this could potentially lead to multiple optimal bundles (e.g., when $U$ is not strictly concave), we add a tie-breaking perturbation to the original utility function to introduce consistency in the buyer's decision making process. That is, we model the buyer's effective utility function as $U'(x) = U(x) + \frac{4}{\mu}\left(\mathlarger \sum \limits_{i=1}^{n} \sqrt{x_i}\right)$, where $\mu$ is a positive constant. There is nothing special about the choice of the tie breaking function, and many other choices can also be used to make the solution unique (for instance, we can use the Cobb-Douglas function as well). The solution $x^*(p)$ is called the induced bundle at prices $p$. 
%Here $\text{dom }U':= \{x \in \mathbb{R}^n \;|\; U'(x) < \infty \}$. 
The seller's goal is to learn the parameters of the function $U(.)$ by observing the bundles bought in a sequence of interactions, where the seller chooses \emph {realistic} prices of items in each interaction. The complexity of any learning algorithm in this setting is typically the number of interactions the seller makes with the buyer to learn the parameters with sufficient accuracy.  

%Assumptions
\noindent \textbf{Assumptions:} We assume that the seller knows the set $C$ of feasible bundles. This is a mild condition, and can be mined from historical purchase data. The set $C \subseteq \text{dom } U'$ is assumed to be non-empty, compact and convex and $\forall x \in C$, $\lVert x \rVert_1 \leq \gamma_1$  and $\lVert x \rVert_2 \leq \gamma_2$ (here $\lVert a \rVert_q$ refer to the $\ell_q$-norm of vector $a$). We further assume that $C = \{ x^TPx + 2q^Tx + r \leq 0 \} $ , where $P \in \mathbf{S}^n, q\in\mathbf{R}^n, r \in \mathbf{R}$ for computational tractability of a program. 
\begin{comment}
(defined later in Equation (\ref{eq:eig_val_prob})).
\end{comment} 
%To theoretically guarantee the accuracy with which our algorithm learns the parameters of the buyer's utility function $U(\cdot)$, we assume the set of feasible bundles, $C$, is defined by the following quadratic constraint:  $C = \{ x^TPx + 2q^Tx + r \leq 0 \} $ , where $P \in \mathbf{S}^n, q\in\mathbf{R}^n, r \in \mathbf{R}$. Note that this assumption is only to ensure "computational tractability" of an optimization program (defined in Equation (\ref{eq:eig_val_prob})) that our learning algorithm needs to solve. 
This does not affect the learning complexity as even without this assumption the program can be solved by performing an exhaustive grid search.

 We also assume that the seller does not know the exact tie breaking parameter $\mu$ that the buyer uses, but knows an upper and lower bound on it i.e., $\mu \in [\mu_1,\mu_2]$. We assume tie-breaking to be the only effect of such a function and that its functional form is known beforehand.

To ensure computational tractability of the buyer's problem in Equation (\ref{eq:Buyer's Problem2}), we make some generic assumptions about the buyer's utility function. Namely, we assume $U(.)$ is concave on the feasible set $C$. Also, let $U(x)$ for each $x \in C$ be non-negative and non-decreasing. Since the tie breaking perturbation $\frac{4}{\mu}\left(\mathlarger \sum \limits_{i=1}^{n} \sqrt{x_i}\right)$ is also non-negative and non-decreasing, so $U'(x)$ non-negative and non-decreasing. 

Further, since $U(x)$ is concave on $C$, $U'(x) = U(x) + \frac{4}{\mu}\left(\mathlarger \sum \limits_{i=1}^{n} \sqrt{x_i}\right)$, is $\frac{1}{\mu}$-strongly concave 
\begin{comment}
(see Proposition \ref{perturbation} in Appendix \ref{sec:appendix:c0})
\end{comment} 
on the set $C$ with respect to $\lVert \cdot \rVert_2$ norm. In other words, the buyer's problem defined in Equation (\ref{eq:Buyer's Problem2}) is a maximization of a strongly concave function over a convex set $C$. Hence $x^*(p)$ exists for every $p \in \mathbb{R}^n$ and is unique (follows from strong concavity). We also assume that the utility function of the buyer, $U(x)$, is $(\lambda_{val},\beta)$-H\"{o}lder continuous 
%(defined in Appendix \ref{sec:appendix:c0}) 
with respect to the $\lVert \cdot \rVert_2$ norm $-$ for all $x,x' \in C$. Thus we have, 
$\lvert U(x) - U(x')\rvert \leq \lambda_{val}\cdot \lVert x - x'\rVert_2^{\beta}.$ Note that this assumption of H\"{o}lder-continuity on the utilities is a mild one and is satisfied by a wide range of economically meaningful utilities like Constant Elasticity of Substitution (CES) and Cobb-Douglas utilities. 

We restrict the scope of our model to utility functions with linear coefficients and known nonlinearities (these are with respect to $x$). This includes many concave utility functions (concave in the bundle) including the CES utility function (this is a function of the form $U(x) = (\sum_{i=1}^{n}\alpha_ix_i^\rho)^{\beta}$ that has linear coefficients when parameter $\beta = 1$ and $\rho < 1$), the logarithm of the Cobb-Douglas function ($\log U(x) = \sum_{i=1}^{n}\alpha_i \log x_i$), and any other function that is approximable by a positive polynomial of bundle $x$. Thus, utilities such as Separable Piecewise-Linear Concave (SPLC), CES, Cobb-Douglas or Leontief functions can also be learned in our setting (although their representation has to be transformed so that the function is linear in the parameters). Later on, without loss of generality, we will assume $U(x)= a^Tx=\sum_{i=1}^{n} a_i \cdot x_i$, with $a \in \mathbb{R}^n_+$. 

\begin{comment}
Note that without an interesting feasible set $C$ of bundles, the learning problem in our setting decomposes into $n$ scalar learning problems that can be solved using binary search. On the other hand, when we have a non-trivial $C$ or a coupling across items through the function $U$, then binary search is no longer applicable.
\end{comment}
 
\subsection{The Learning Algorithm}
\noindent\textbf{Overview:}
Without an interesting feasible set $C$ of bundles, the learning problem in our setting can decompose into $n$ scalar learning problems that can be solved using binary search. On the other hand, when we have a non-trivial $C$ or a coupling across items through the function $U$, then binary search is no longer applicable. The algorithm that we propose for learning the unknown parameter vector $a^*$ is based on maintaining uncertainty ellipsoids around $a^*$ and successively shrinking their volume by constructing specific separating hyperplanes (based on observed purchases). At each round $t$, we start with an uncertainty ellipsoid $E_{t}$ and shrink it to get $E_{t+1}$. In particular, based on the interaction between the buyer and the seller in the current round, we cut $E_{t}$ with a hyperplane into two regions. And then we update $E_{t+1}$ as the L\"{o}wner-John ellipsoid of one of these regions.

The main technical part of our algorithm is that it works by seeking a desired purchase vector in each round. The purchase vector is then used to deduce a hyperplane that cuts the uncertainty set. Now, this purchase vector cannot be directly accessed as we can only control prices to induce purchase. Below, we show how to use gradient descent and duality to find prices that induce desired bundles.

Along with the price that induces desired bundles, we are able to get the value of these bundles. We compare these values with the minimum and maximum values that are possible given our current uncertainty set over parameter vector $a^*$ and define appropriate hyperplanes to split the uncertainty sets, thus shrinking them.\\

\noindent\textbf{Finding a price that induces a specific bundle:} Consider the following convex program:
%\vspace*{-2mm}
\begin{gather}
\label{eq:reformulated_buyer}
\begin{array}{ll@{}ll}
	 \max_{\mathbf{x} \in C}  &U'(x)&\\
	\text{s.t} & x_j \leq \widehat{x}_j \quad \text{for every item $j \in [n]$},\\
\end{array}
%\vspace*{-0.5mm}
\end{gather}
where $\hat{x} \in C$ is a specific bundle.
Since the utility function $U'(x)$ is non-decreasing and $1/\mu$-strongly concave, we can see that $\widehat{x}$ is the unique optimal solution of the problem in Equation (\ref{eq:reformulated_buyer}). The partial Lagrangian of this formulation (\ref{eq:reformulated_buyer}) is defined as:
$\mathcal{L}(x,p) = U'(x) - p^Tx + p^T\widehat{x}$,
where $p \in \mathbb{R}^n_+$ is the dual variable. We define the Lagrange dual function $g:\mathbb{R}^n \rightarrow \mathbb{R}$ to be $g(p) = \max \limits_{x \in C} \mathcal{L}(x,p) = \max\limits_{x \in C} U'(x) - p^Tx + p^T\widehat{x}$.
Now the dual of the convex program in Equation (\ref{eq:reformulated_buyer}) can be defined as:
\begin{gather}
\label{eq:reformulated_buyer_dual}
\begin{array}{ll@{}ll}
	 \min &g(p)&\\
	\text{s.t} & p \in \mathbb{R}^n_+.\\
\end{array}
\end{gather}

Our algorithm needs to choose a specific bundle $\widehat{x}$ and learn its value $U'(\widehat{x})$. Since we can only control prices, we show how to learn the value $U'(\widehat{x})$ by working with the dual problem. In other words, to compute $U'(\widehat{x})$, which we otherwise could not have since $U'(.)$ is unknown, we define the problem in (\ref{eq:reformulated_buyer}) such that its optimal solution is $\widehat{x}$ itself. We can compute the minimizer $\widehat{p}$ of its dual in (\ref{eq:reformulated_buyer_dual}) because we can control prices. And, by strong duality, we will get $g(\widehat{p}) = U'(\widehat{x}) = \text{OPT}$. 

Now we focus on the problem of minimizing the function $g$, which is also unknown (since $U'(.)$ is unknown). However, due to the structure of the dual problem, the function $g(p)$ can be approximately optimized using a first order optimization technique such as projected gradient descent. 
In particular, this is the structure we exploit: we have access to the gradients of $g$ and these turn out to be functions of $\widehat{x}$ and the actual bundles purchased by the buyer. Thus, we can set a price $p$, interact with the buyer to observe the bundle purchased $x^*(p)$ and get access to the gradient. Formally, the following Lemma \ref{lemma:subgradient} shows that the bundle $x^*(p)$ purchased by the buyer gives a gradient of the Lagrange dual function $g(\cdot)$ at $p$. 

\begin{lemma}\label{lemma:subgradient}
Since the convex program in Equation (\ref{eq:reformulated_buyer}) has a unique optimal solution, therefore $g(p)$ is differentiable at each $p \in \mathcal{P}$. Moreover, if a price vector $p$ induces bundle $x^*(p)$, then the gradient of $g(p)$ at $p$ is given by $\nabla g(p) = \widehat{x} - x^*(p).$
\end{lemma}
\begin{comment}
\begin{proof}
Since the utility function $U'(x)$ is strongly concave over the set $C$ of feasible bundles, therefore the maximizer, $x' = \text{arg} \max \limits_{x \in C} U'(x) - p^Tx + p^T\widehat{x}$, exists and is unique for every $p \in \mathbb{R}^n$. Now since the set $C$ is non-empty and compact, and $U'(x)$ and $\widehat{x} - x$ are continuous functions over $C$, by Theorem $6.3.3$ in \cite{bazaraa2013nonlinear}, the dual function $g$ is differentiable at every $p \in \mathbb{R}^n$. Further its gradient $\nabla g(p) = \widehat{x} - x'$. It is easy to see that $x'$ is the bundle purchased by the buyer at price $p$ because:
\begin{equation*}
x' = \text{arg} \max \limits_{x \in C} U'(x) - p^Tx + p^T\widehat{x} = \text{arg} \max \limits_{x \in C} U'(x) - p^Tx = x^*(p).
\end{equation*}
\end{proof}
\end{comment}

Next, we focus on the restriction to realistic prices. We are constrained to set prices only from the realistic price space $\mathcal{P}$, so we can only solve a restricted version of the dual program in Equation (\ref{eq:reformulated_buyer_dual}), which we denote as $\min\limits_{p\in \mathcal{P}} g(p).$
The following Lemma shows that instead of minimizing $g(p)$ in Equation (\ref{eq:reformulated_buyer_dual}) over $p\in \mathbb{R}^n_+$, if it is minimized over the \textit{realistic} price space $\mathcal{P}$ as defined in Definition \ref{def:realistic-price}, then the optimal value remains close to OPT. 

\begin{lemma}
\label{lemma:lemma_from_STOC16_adaptation}
There exists a value R-OPT such that $\min\limits_{p\in \mathcal{P}} g(p) = \text{R-OPT}.$
Moreover, $U'(\widehat{x}) \leq \text{R-OPT} \leq U'(\widehat{x}) + \tau$, where 
\begin{equation}
\label{eq:tau}	
\tau = \max \left\lbrace \lambda_{val}\left( \frac{2\underline{L} \gamma_1}{\overline{L}} \right)^\beta,\;  \lambda_{val}^{\frac{1}{1-\beta}} \left(\frac{2}{\overline{L}}  \right)^{\frac{\beta}{1-\beta}}  \right\rbrace + \underline{L}\gamma_1, 
\end{equation}
$\overline{L} = \lVert p^0 + \Delta \rVert_{-\infty}$ and $\underline{L} = \lVert p^0 - \Delta \rVert_\infty$ ($p^0$ and $\Delta$ defined in Definition \ref{def:realistic-price}).
\end{lemma}
%\begin{comment}
\begin{proof}
The sets $C$ and $\mathcal{P}$ are convex. And $\mathcal{P}$, the realistic price space defined in Definition \ref{def:realistic-price}, is also closed, compact and bounded since it is an $n$-orthotope as shown in Equation (\ref{eq:realistic_price2}). Therefore, by the minimax theorem \cite{sion1958}, there exists a value $\textit{R-OPT}$ such that 
\begin{equation}
\label{eq:minimax}
\max \limits_{x \in C} \min \limits_{p \in \mathcal{P}} \mathcal{L}(x,p) = \min \limits_{p \in \mathcal{P}} \max \limits_{x \in C} \mathcal{L}(x,p) = \min\limits_{p\in \mathcal{P}} g(p) = \text{R-OPT} 
\end{equation}
Moreover, $\textit{R-OPT} = \min\limits_{p\in \mathcal{P}} g(p)$ is the optimum solution of a restriction of the dual formulated in Equation (\ref{eq:reformulated_buyer_dual}), as $\mathcal{P} \subseteq \mathbb{R}^n_+$, hence $\textit{R-OPT} \geq U(\widehat{x})$. So what remains to be shown is $\text{R-OPT} \leq U(\widehat{x}) + \tau $. Let $(\tilde{x}, \tilde{p})$ be a pair of minimax strategies for Equation (\ref{eq:minimax}). That is,
\begin{displaymath}
\tilde{x} \in \text{arg} \max \limits_{x \in C} \min \limits_{p \in \mathcal{P}} \mathcal{L}(x,p) \;\text{ and }\; \tilde{p} \in \text{arg}\min \limits_{p \in \mathcal{P}} \max \limits_{x \in C} \mathcal{L}(x,p). 
\end{displaymath}
Now, by strong duality we have,
\begin{displaymath}
\textit{R-OPT} = U(\tilde{x}) - \tilde{p}^T(\tilde{x} - \widehat{x}).
\end{displaymath}
Choosing a price vector $p' \in \mathcal{P}$ such that
\[
    p'_j= 
\begin{cases}
    p^0_j + \delta_j,& \text{if } \tilde{x}_j > \widehat{x}_j, \textrm{ and}\\
    p^0_j - \delta_j,& \text{if } \tilde{x}_j \leq \widehat{x}_j,
\end{cases}
\]
gives
\begin{gather}
\begin{split}	
\textit{R-OPT} &\leq U(\tilde{x}) - {p'}^T(\tilde{x} - \widehat{x}) \leq U(\tilde{x}) - \sum \limits_{j: \tilde{x}_j > \widehat{x}_j} \overline{L}(\tilde{x}_j - \widehat{x}_j) \\
& +\; \sum \limits_{j: \tilde{x}_j < \widehat{x}_j} \underline{L}(\widehat{x}_j - \tilde{x}_j ),
\end{split}
\end{gather}
where $\overline{L}$ and $\underline{L}$ are as defined in Lemma \ref{lemma:lemma_from_STOC16_adaptation}.
Now consider the bundles $y$ and $z$ such that $y_j= \max\{\tilde{x}_j, \widehat{x}_j\}$ and $z_j = \min\{\tilde{x}_j, \widehat{x}_j\}$ for all $j \in [n]$. Since $U(.)$ is an increasing function so $U(y) \geq U(\tilde{x})$, and $U(.)$ is also assumed to be $(\lambda_{val}, \beta)$-H\"{o}lder continuous with respect to $\ell_2$-norm, so we have
\begin{displaymath}
U(\tilde{x}) - U(\widehat{x}) \leq U(y) - U(\widehat{x}) \leq \lambda_{val} \lVert y - \widehat{x}\rVert^{\beta}_2. 
\end{displaymath} 
Therefore,
\begin{align}
%\begin{split}
\label{eq:R-OPT}
\textit{R-OPT} &\leq U(y) - \sum \limits_{j: \tilde{x}_j > \widehat{x}_j} \overline{L}(y_j - \widehat{x}_j) + \sum \limits_{j: \tilde{x}_j < \widehat{x}_j} \underline{L}(\widehat{x}_j - z_j ) \notag \\
&= U(y) - \overline{L} \lVert y - \widehat{x} \rVert_1 + \underline{L} \lVert \widehat{x} - z \rVert_1 \notag \\
&\leq U(y) - \overline{L} \lVert y - \widehat{x}\rVert_2 + \underline{L} \lVert \widehat{x} - z\rVert_1 \notag \\
&\leq U(\widehat{x}) + \lambda_{val} \lVert y - \widehat{x}\rVert^{\beta}_2 - \overline{L} \lVert y - \widehat{x}\rVert_2 + \underline{L} \lVert \widehat{x} - z\rVert_1 \notag \\
&= U(\widehat{x}) + \notag \\
&\lambda_{val} \lVert y - \widehat{x}\rVert^{\beta}_2 \left\lbrace 1 - \frac{\overline{L}}{\lambda_{val}}\lVert y -\widehat{x} \rVert_2^{1-\beta} + \frac{\underline{L}}{\lambda_{val}} \frac{\lVert \widehat{x} - z\rVert_1}{\lVert y - \widehat{x} \rVert_2^{\beta}}\right\rbrace.  
%\end{split}
\end{align}

Now $\left\lbrace 1 - \frac{\overline{L}}{\lambda_{val}}\lVert y -\widehat{x} \rVert_2^{1-\beta} + \frac{\underline{L}}{\lambda_{val}} \frac{\lVert \widehat{x} - z\rVert_1}{\lVert y - \widehat{x} \rVert_2^{\beta}}\right\rbrace \geq 0$, as otherwise $\textit{R-OPT} < U(\widehat{x})$.
By substituting, $t = \lVert y - \widehat{x}\rVert_2$, $c_1 = \frac{\overline{L}}{\lambda_{val}}$, $c_2 = \frac{\underline{L}}{\lambda_{val}}\gamma_1$ and using $\lVert \widehat{x} - z\rVert_1 \leq \gamma_1$, we have a polynomial $p(t) = t^\beta - c_1t + c_2$, such that $p(t)\geq 0$.  Our immediate goal is to get an upper bound on the range in which the positive real roots of $p(t)$ lie. Let us make a mild assumption that $\beta = \frac{p}{q}$ with $p,q \in \mathbb{Z}_+$, and $q > p$ (since $\beta \in (0,1]$), i.e., it is rational. With the transformation $t^{1/q}=s$, we have $p(s) = c_1s^q - s^p -c_2$, and $p(s) \leq 0.$ Since $q>p$, so the degree of the polynomial $p(s)$ is $q$. Also $c_1$ is positive as $\lambda_{val} \geq 0$
%doublecheck \geq 0 or \geq 1!
and $\overline{L} > 0$ as $\mathcal{P} \neq \{0\}$. Therefore, $p(s)$ is an increasing polynomial, hence we can claim $\max \limits_{s} \{ s \in \mathbb{R}_+ \;|\;p(s) \leq 0\} \leq \max \limits_{s} \{ s\in \mathbb{R}_+\;|\;p(s) = 0\}$. 

To get an upper bound on $\lVert y - \widehat{x} \rVert_2$, i.e., an upper bound on $\max \limits_{t} \{ t\in \mathbb{R}_+\;|\;p(t) \geq 0\}$, it suffices to upper bound the positive real roots of the polynomial equation $p(s) = 0$. Note that, using the \textit{Descartes' rules of sign}, the polynomial $p(s)$ has exactly one positive real root, which can be upper bounded using Cauchy's theorem \cite{obreshkov1963verteilung} 
\begin{comment}
(refer to Appendix \ref{sec:appendix:d}) 
\end{comment}
as follows:
\begin{equation*}
\max \limits_{s} \{ s\in \mathbb{R}_+\;|\;p(s) = 0\} \leq \max \left\lbrace \left( \frac{2c_2}{c_1} \right)^{1/q}, \left( \frac{2}{c_1} \right)^{1/q-p}  \right\rbrace.  
\end{equation*}
Hence,
\begin{equation*}
\lVert y - \widehat{x} \rVert_2 \leq \max \left\lbrace \left( \frac{2\underline{L} \gamma_1}{\overline{L}} \right), \left( \frac{2\lambda_{val}}{\overline{L}} \right)^{\frac{1}{1-\beta}}  \right\rbrace.  
\end{equation*}
Now, the inequality in Equation (\ref{eq:R-OPT}) becomes
\begin{align*}
\textit{R-OPT} &\leq U(\widehat{x}) + \lambda_{val} \lVert y - \widehat{x}\rVert_2^{\beta}  + \underline{L} \; \lVert \widehat{x} - z\rVert_1 \\
&\leq U(\widehat{x}) + \max \left\lbrace \lambda_{val}\left( \frac{2\underline{L} \gamma_1}{\overline{L}} \right)^\beta,\;  \lambda_{val}^{\frac{1}{1-\beta}} \left(\frac{2}{\overline{L}}  \right)^{\frac{\beta}{1-\beta}}  \right\rbrace + \underline{L}\gamma_1   .
\end{align*}
\end{proof}
%\end{comment}

The dual function $g(p)$ is convex, and the following Lemma \ref{lemma:frenchel-conjugate} further shows that $g(p)$ is also strongly smooth. 
\begin{comment}
(defined in Appendix \ref{sec:appendix:c0}). 
\end{comment}
\begin{lemma}
\label{lemma:frenchel-conjugate}
The function g(p) is $\mu$-strongly smooth with respect to the $\lVert \cdot \rVert_2$ norm.
\end{lemma}
\begin{comment}
\begin{proof}
Note that
\begin{align*}
g(p) &= \max_{x \in C} \; U'(x) - p^Tx + p^T\widehat{x} \\
&= \max_{x \in C} \; \{(-p)^Tx - (-U')(x)\} +  p^T\widehat{x} = (-U')^*(-p) + p^T\widehat{x},
\end{align*}
where $(-U')^*(.)$ is the Fenchel conjugate of $(-U')(.)$ and for every $p \in \mathbb{R}^n$, $-p \in \text{dom }(-U')^* $ as $\max_{x \in C}\; U'(x) - p^Tx$ is finite. Now, by Corollary 3.5.11 in \cite{zalinescu2002convex}, since $-U'$ is $\frac{1}{\mu}$-strongly convex with respect to $\lVert \cdot \rVert_2$ norm, therefore $(-U')^*(.)$ is $\mu$-strongly smooth with respect to the $\lVert \cdot \rVert_2$norm. Hence, $g(p)$ is $\mu$-strongly smooth with respect to the $\lVert \cdot \rVert_2$norm.  
\end{proof}
\end{comment}

Convexity and smoothness of $g(p)$ are useful below, where we give a projected gradient descent 
\begin{comment}
(refer to Appendix \ref{sec:appendix:c} for a brief description) 
\end{comment}
procedure \textproc{learnvalue$(\widehat{x}, \tau)$} (Algorithm \ref{alg:pgd}). Given a target bundle $\widehat{x} \in C$ and an error budget $\tau$ (this is the same value appearing in Lemma~\ref{lemma:lemma_from_STOC16_adaptation}), \textproc{learnvalue$(\widehat{x}, \tau)$} minimizes $g(p)$ over the realistic price space $\mathcal{P}$ defined in Definition \ref{def:realistic-price}, with an additive error of at most $\tau$.  
\begin{comment}
In the projected gradient descent procedure, we choose the step size as $\eta = \frac{1}{\mu_2}$. Note that here we use the upper bound $\mu_2$ on $\mu$ and since $g$ is $\mu$-strongly smooth (Lemma \ref{lemma:frenchel-conjugate}) so it is also $\mu_2$-strongly smooth. 
\end{comment}
\begin{algorithm}
    \caption{Solving the Lagrangian Dual}
    \label{alg:pgd}
    \begin{algorithmic}[1] % The number tells where the line numbering should start
        \Procedure{learnvalue}{$\widehat{x}$,$\tau$} 
            \State Initialize: $p_1$ and $T = \frac{50\gamma_2\mu_2}{\tau - R^2\gamma_2}$. %= p^0
            \For{$t = 1, \cdots, T$}
        		\State Observe the purchased bundle, $x^*(p_t)$,by the buyer.
                \State Update the price vector with projected gradient descent: 
                \begin{equation*}
                p_{t+1} = \prod_\mathcal{P}\;\Bigg[ p_t - \eta_t (\widehat{x} - x^*(p_t))\Bigg],
                \end{equation*}
            	\State where $\eta_t = \gamma/\lVert \nabla g(p_t)\rVert$, and $\gamma = 1/T$
      		\EndFor
            \State \textbf{return} $
            \tilde{g}(p_T)= g(p_1) + \sum \limits_{t=1}^{T-1} \nabla g(p_t) (p_{t+1} - p_{t}) +
            \frac{\mu_2}{2} \; \mathlarger \sum \limits_{t=1}^{T-1} \; \lVert p_{t+1} -p_t\rVert_2^2.$
            
        \EndProcedure
    \end{algorithmic}
\end{algorithm}
%\vspace*{-1mm}
\begin{theorem}{(Main Supporting Result)}
\label{lemma:compute g(p)} Assuming $g(p_1)$ is known and that $\tau \geq R^2 \gamma_2$, \textproc{learnvalue$(\widehat{x}, \tau)$} (Algorithm \ref{alg:pgd}) can estimate R-OPT to accuracy $\tau$. That is after $T = \frac{50\gamma_2\mu_2}{\tau - R^2\gamma_2}$ interactions with the buyer, 
\begin{equation*}
\tilde{g}(p_{_T}) - \text{R-OPT} \leq \tau,
\end{equation*}
where $\tilde{g}(p_{_T})$ is the estimate of R-OPT returned by \textproc{learnvalue$(\widehat{x}, \tau)$}.  
\end{theorem}
%\begin{comment}
\begin{proof}[Proof Sketch]
The value of $g$ at each each of the subsequent iterates of the projected gradient procedure can be approximated using Lagrange first order approximation. Thus,
\begin{equation*}
g(p_{t+1}) = g(p_t) + \nabla g(p_t) (p_{t+1} - p_{t}) + \mathcal{E}_{t+1}, \;\; t \in [T-1] ,
\end{equation*}
where $\mathcal{E}$ is the Lagrangian error. Therefore, adding the values of $g$ at each iteration, the sum telescopes and we get
\begin{align*}
g(p_{_T}) &= g(p_{_T})^{'} + \mathlarger \sum \limits_{t=1}^{T-1} \mathcal{E}_{t+1},   
\end{align*}
where $g(p_{_T})^{'}= g(p_1) + \mathlarger \sum \limits_{i=1}^{T-1} \nabla g(p_t) (p_{t+1} - p_{t})$. 
\begin{comment}is the estimate of $g(p_{_T})$ returned by \textproc{learnvalue$(\widehat{x}, \epsilon)$} (Algorithm \ref{alg:pgd}). 
\end{comment}
Thus,
\begin{equation}
\label{eq:er1}
\lvert g(p_{_T}) - g(p_{_T})^{'}\rvert = \lvert \mathlarger \sum \limits_{t=1}^{T-1} \mathcal{E}_{t+1} \rvert \leq \mathlarger \sum \limits_{t=1}^{T-1} \lvert \mathcal{E}_{t+1}\rvert. 
\end{equation}
Now using Taylor's remainder theorem 
\begin{comment}
(Theorem \ref{thm:taylor} in Appendix \ref{sec:appendix:f}) 
\end{comment}
and the fact that $g$ is $\mu_2$-strongly smooth,  $\lVert \nabla^2 g(p)\rVert_{\text{max}} \leq \lVert \nabla^2 g(p)\rVert_2 \leq \mu_2$ for all $p \in \mathcal{P}$, we have
\begin{equation}
\label{eq:er2}
\lvert \mathcal{E}_{t+1}\rvert \leq \frac{\mu_2}{2} \lVert p_{t+1} -p_t\rVert_2^2, \;\; \forall i=1,..., T-1.
\end{equation}
Moreover, using Equation (\ref{eq:er2}) in Equation (\ref{eq:er1}), and the fact that since $g(.)$ is convex therefore the first order lagrange approximation is a global under estimator, we get
\begin{equation}
\label{eq:error in g(p)}
g(p_{_T}) - g(p_{_T})^{'} \leq \frac{\mu_2}{2} \; \mathlarger \sum \limits_{t=1}^{T-1} \; \lVert p_{t+1} -p_t\rVert_2^2. 
\end{equation}
Plugging Equation (\ref{eq:error in g(p)}) in 
\begin{comment}
the following
\begin{equation}
g(p_{_T}) - \text{R-OPT} \leq \frac{\lVert p^* - p_1 \rVert_2^2 + T\gamma^2}{2\gamma\sum_{t=1}^{T}\frac{1}{\lVert\nabla g(p_t)\rVert_2}},
\end{equation}
which is the 
\end{comment}
the guarantee for projected gradient descent 
\begin{comment}
(Theorem \ref{thm:pgd} in Appendix \ref{sec:appendix:c}) 
\end{comment}
we get:
\begin{gather}
\begin{split}
\label{eq:pgd guarantee}
g(p_{_T})^{'} + \frac{\mu_2}{2} \; \mathlarger \sum \limits_{t=1}^{T-1} \; \lVert p_{t+1} -p_t\rVert_2^2 \leq \text{R-OPT} \;\;+ \\ 
\; \frac{\lVert p^* - p_1 \rVert_2^2 + T\gamma^2}{2\gamma\sum_{t=1}^{T}\frac{1}{\lVert\nabla g(p_t)\rVert_2}} + \frac{\mu_2}{2} \; \mathlarger \sum \limits_{t=1}^{T-1} \; \lVert p_{t+1} -p_t\rVert_2^2 .
\end{split}
\end{gather}
Therefore, in the Algorithm \ref{alg:pgd}, by choosing constant step lengths in the projected gradient descent procedure i.e., $\eta_t = \gamma/\lVert \nabla g(p_t)\rVert$, where $\gamma = 1/T$ we get $\lVert p_{t+1}-p_t\rVert=\gamma$ for each $t=[T-1]$.
Now, by assuming $\lVert\nabla g(p_t)\rVert_2 = \lVert \widehat{x} - x^*(p)\rVert_2\leq \gamma_2$ and $\lVert p^* - p_1\rVert \leq R$, Equation (\ref{eq:pgd guarantee}) becomes:

\begin{displaymath}
g(p_{_T})^{'} + \frac{\mu_2}{2} \; \mathlarger \sum \limits_{t=1}^{T-1} \; \lVert p_{t+1} -p_t\rVert_2^2 \leq \text{R-OPT} + \frac{R^2\gamma_2}{2} + \frac{\gamma_2}{2T} + \frac{\mu_2}{2T}.
\end{displaymath}

Note that $\tau$ as defined in Equation (\ref{eq:tau}) is greater than $R^2 \gamma_2$ by assumption. Thus, after at least $T \geq \frac{\gamma_2 + \mu_2}{2(\tau-R^2\gamma_2)}$ iterations the Algorithm \ref{alg:pgd} produces a $\tau$-optimal solution, $\tilde{g}(p_{_T})$.
\end{proof}
%\end{comment}

\begin{comment}
\begin{theorem}
\label{thm:theorem_for_LearnVal}
Algorithm \textproc{learnvalue} outputs $g(p_{_T})^{'}$,  such that: $g(p_{_T})^{'} \leq \min\limits_{p \in \mathcal{P}} g(p) + \tau$,
and the number of rounds it needs to interact with the buyer is at most 
$T = O(\frac{1}{\tau})$.
\end{theorem}
\begin{proof}
must fill the proof here
\end{proof}
\end{comment}

Therefore, combining Lemma \ref{lemma:lemma_from_STOC16_adaptation} and Theorem \ref{lemma:compute g(p)}, we have:
\begin{equation}
\label{eq:error_in_solving_for_value_U(x)}
U'(\widehat{x}) \leq \tilde{g}(p_{_T}) \leq U'(\widehat{x}) + 2\tau .
\end{equation}

Hereafter in this section, for the sake of simplicity of illustration, we focus on learning the buyer's utility function $U(x)= a^{*T}x=\sum_{i=1}^{n} a_i^* \cdot x_i$ assuming it is linear in both the coefficients and the bundle (this is without loss of generality). Hence Equation (\ref{eq:error_in_solving_for_value_U(x)}) becomes:

\begin{equation}
\label{eq:error_in_solving_for_value}
a^{*T}\widehat{x} \leq \tilde{g}(p_{_T}) \leq a^{*T}\widehat{x} + \frac{4}{\mu_1}\left(\mathlarger \sum \limits_{i=1}^{n} \sqrt{\widehat{x}_i}\right) + 2\tau  
\end{equation}

\noindent\textbf{Interval containing the value $U(\widehat{x})$:} 
\begin{comment}
Before we give our algorithm that, as part of its procedure, computes a halfspace to cut the uncertainty ellipsoid $E_{t} = E(A_{t},c_{t})$ (see Appendix~\ref{sec:appendix:a}), we give one final Lemma \ref{lemma:grotschel}. 
\end{comment}
It turns out that for a target bundle $\widehat{x}$, that the seller has in mind, she can compute an interval $[\,\underline{b}_t, \overline{b}_t\,]$ using the uncertainty ellipsoid $E_{t}$ such that it contains the scalar value $\widehat{x}^T{a^*}$. Lemma \ref{lemma:grotschel} gives the optimum values of the following convex  programs:
\begin{displaymath}
\underline{b}_t = \min\limits_{\tilde{a} \in E(A,c)} \widehat{x}^T\tilde{a}, \quad \text{and}\;\overline{b}_t = \max\limits_{\tilde{a} \in E(A,c)} \widehat{x}^T\tilde{a}. 
\end{displaymath}

\begin{lemma}
\label{lemma:grotschel}
\cite{grotschel2012geometric}
For any $\widehat{x} \in \mathbb{R}^{n} \setminus \{0\}$,
\begin{displaymath}
\text{arg}\max\limits_{\tilde{a} \in E(A,c)} \widehat{x}^T\tilde{a} = c + b,\quad
\text{arg}\min\limits_{\tilde{a} \in E(A,c)} \widehat{x}^T\tilde{a} = c - b, 
\end{displaymath}
where $b = A\widehat{x}/ \sqrt{\widehat{x}^TA\widehat{x}}.$
\end{lemma}
So, if $g(p_T)' \leq (\underline{b} + \overline{b})/2 = \widehat{x}^Tc$, then the unknown parameter $a^*$ lies in the halfspace 
\begin{equation}
\label{eq:H_t_less}
H = \{\tilde{a} \in \mathbb{R}^n: \widehat{x}^T\tilde{a} \leq \widehat{x}^{T}c\}.
\end{equation}

On the other hand if $g(p_T)' \geq (\underline{b} + \overline{b})/2 = \widehat{x}^Tc$, then by Equation (\ref{eq:error_in_solving_for_value}), the unknown parameter $a^*$ lies in the halfspace
\begin{equation}
\label{eq:H_t_more}
H = \left\{\tilde{a} \in \mathbb{R}^n: \widehat{x}^T\tilde{a} \geq \widehat{x}^{T}c - \left(\frac{4}{\mu_1}\left(\mathlarger \sum \limits_{i=1}^{n} \sqrt{\widehat{x}_i}\right) + 2\tau\right) \right\}.
\end{equation}
\begin{comment}
\begin{proposition}
\label{prop:on bundle space}
The corresponding hyperplane in (\ref{eq:H_t_more}) intersects the ellipsoid $E(A,c)$ if $\lVert A^{1/2} \widehat{x}\rVert_2 > \frac{\epsilon^2\sigma}{2}.$
\end{proposition}
\end{comment}

% \noindent\textbf{The main algorithm:} Now we  present the \algoLUU algorithm (Algorithm \ref{alg:stackelberg}) for learning the parameter $a^*$ of the buyer's utility function. 
\begin{algorithm}
    \caption{Learning Utility Maximizing Buyer's Model}
    \label{alg:stackelberg}
    \begin{algorithmic}[1] % The number tells where the line numbering should start
        \Procedure{\algoLUU}{$\epsilon$} 
            \State $E_0 = E(A_0,c_0) \subseteq \mathbb{R}^{n}$ is the initial uncertainty ellipsoid with $A_0 = R_a\cdot I$ for $R_a > 0$ as defined in Theorem~\ref{thm:Algo_iter}.
            \State Pick bundle $x_t = \text{arg} \max\limits_{x \in \mathcal{C}} \sqrt{x^TA_0x}$
            \Do  
                %\State Compute $\underline{b_t}$ and $\overline{b_t}$ using Lemma \ref{lemma:grotschel}.
                \State $\tilde{g}(p_{_T}) \gets$ \textproc{LearnValue$(x_t, \tau)$}
                \State \textbf{if} $\tilde{g}(p_{_T}) \leq (\underline{b_t} + \overline{b_t})/2 = x_t^Tc$ \textbf{then} $H_t$ is (\ref{eq:H_t_less}),
                \State \textbf{else} $H_t$ is (\ref{eq:H_t_more}).
                \State $E_{t+1} = \textrm{LJohn}(E_{t} \cap H_t)$ (here LJohn() finds the L\"{o}wner-John ellipsoid of its argument).
                \State Pick bundle $x_t = \text{arg} \max\limits_{x \in \mathcal{C}} \sqrt{x^TA_{t+1}x}$
            \doWhile{$(2\sqrt[]{x_t^T A_{t+1} x_t} > \epsilon)$}%\label{euclidendwhile}
            %\State \textbf{return} $b$
            %ISSUE WITH EPSILON and EPSILON' ABOVE!!!!
        \EndProcedure
    \end{algorithmic}
\end{algorithm}

\vspace*{-2mm}
\subsection{Analysis}

Note that in \algoLUU (Algorithm \ref{alg:stackelberg}), the uncertainty ellipsoid $E_{t+1}$ for the next iteration is updated using the computation 
$E_{t+1} = \textrm{LJohn}(E_{t} \cap H_t)$, where $H_t$ is defined by either Equation (\ref{eq:H_t_less}) or (\ref{eq:H_t_more}). The former induces a central cut in the ellipsoid $E_t(A, c)$, i.e. the hyperplane $H_t$ passes through the center $c$ and eliminates half of the volume of the ellipsoid. On the other hand, the later hyperplane induces a shallow cut and removes less than half of the volume. Without loss of generality (as we only need an upper bound on the number of iterations needed by \algoLUU to learn $a^*$), we assume that at each iteration the relevant hyperplane induces a shallow cut. That is, $H_t$ is:
\begin{equation}
\label{eq:H_t_more_less}
H_t = \{\tilde{a} \in \mathbb{R}^n: x_t^T\tilde{a} \lessgtr x_t^{T}c \pm \delta\} 
\end{equation}
where $\delta = \left(\frac{4}{\mu_1}\left(\mathlarger \sum \limits_{i=1}^{n} \sqrt{{x_t}_i}\right) + 2\tau\right)$ is the depth of the cut induced.
For \algoLUU to work we need the depth $\delta$ to be at most $\frac{\sqrt[]{x_t^T A_t x_t}}{n}$, i.e. $\delta \leq \frac{\sqrt[]{x_t^T A_t x_t}}{n}$. As the portion $\frac{4}{\mu_1}\left(\mathlarger \sum \limits_{i=1}^{n} \sqrt{{x_t}_i}\right)$ takes effect only in tie-breaking, i.e., we can assume $\mu_1$ to be a large constant. Hence, the constraint on the depth of the shallow cut becomes
$\sqrt{x_t^TA_tx_t} \geq 2n\tau$. Also, note that the Algorithm \algoLUU continues as long as $2\sqrt{x_t^TA_tx_t} > \epsilon$. So the shallow cut condition is met (in other words the algorithm is able to find an $x_t$ in each iteration) as long as $\tau < \epsilon/4n$. 

Next, the computation of the L\"{o}wner-John ellipsoids of the sets that remain after shallow cuts follows from~\cite{grotschel2012geometric}. The L\"{o}wner-John ellipsoid of the set $E_t(A_t,c_t) \cap \{\tilde{a} \in \mathbb{R}^n: x_t^T\tilde{a} \leq x_t^{T}c + \delta\}$ is
$E(A_{t+1}, c_t - \frac{1+n\alpha_t}{n+1}b_t)$, and of the set $E_t(A_t,c_t) \cap \{\tilde{a} \in \mathbb{R}^n: x_t^T\tilde{a} \leq x_t^{T}c - \delta\}$ is $E(A_{t+1}, c_t + \frac{1+n\alpha_t}{n+1}b_t)$, where $\alpha_t = - \frac{\delta}{\sqrt[]{x_t^T A_t x_t}}$, $b_t = A_t x_t/ \sqrt[]{x_t^T A_t x_t}$ and $A_{t+1} = \frac{n^2}{n^2 - 1}(1 - \alpha^2)(A_t - \frac{2(1 + n\alpha)}{(n+1)(1+\alpha)}b_tb_t^T).$ 
\begin{comment}
(defined in Appendix~\ref{appendix:ellipsoid_halfspace(2)}).
\end{comment}

In what follows we present the performance guarantee of the Algorithm \ref{alg:stackelberg}. Firstly, to bound the minimum eigenvalue $\lambda_n$ at successive iterations of our algorithm, we give the following two lemmas from \cite{cohen2016feature} also applicable in our setting. In~\cite{cohen2016feature}, they are used in the analysis of a different algorithm in a different setting (regret minimization). 

\begin{lemma}
	\label{cohen_lemma4}
	For any iteration step $t$, we have $\lambda_n(A_{t+1}) \geq \frac{n^2}{(n+1)^2} \lambda_n(A_t)$.
\end{lemma}

\begin{lemma}
	\label{cohen_lemma5}
	There exists a sufficiently small $k = k(n)$ such that if $\lambda_n(A_t) \leq k\epsilon^2$ and $x_t^TA_tx_t > \frac{1}{4}\epsilon^2$, then $\lambda_n(A_{t+1}) \geq \lambda_n(A_t)$, i.e., the smallest eigenvalue doesn't decrease after the update. One can assume $k = \frac{1}{400n^2}$.
\end{lemma}
  
Using the above two lemmas, we can show that the number of rounds needed by \algoLUU is upper bounded.

\begin{theorem}
\label{thm:Algo_iter}
The algorithm \algoLUU terminates after at most $20n^2 \ln \left(\frac{20R_a(n+1)}{\epsilon}\right)$ iterations, where $R_a$ is the radius of the initial uncertainty set $E_0$. 
\end{theorem}

Combining Theorem~\ref{lemma:compute g(p)} and ~\ref{thm:Algo_iter}, we get the following bound on the interactions needed to get a tight uncertainty set around the unknown parameter $a^*$ of the buyer's utility function. Moreover, since the volume of the uncertainty sets decrease in successive iterations so having a bound on how much the minimum eigenvalue can decrease in one iteration can guarantee the tightness of the uncertainty set at termination, or alternatively, its maximum eigenvalue (for a special case: Corollary \ref{corr:our corr}).  

\begin{theorem}{(Main Result)}
\label{thm:stackelberg_utility_learn}
Assume that the feasible set $C$, the realistic price set $\mathcal{P}$ and the algorithm parameter $\epsilon$ obey the condition: $R^2\gamma_2 \leq \tau \leq \frac{\epsilon}{4n}$. Then, after at most $t \cdot T$ interactions with the buyer, where $t= 20n^2 \ln \left(\frac{20R_a(n+1)}{\epsilon}\right)$ and $T = \frac{50\gamma_2\mu_2}{\tau - R^2\gamma_2}$, algorithm \algoLUU outputs uncertainty set $E(A_t, c_t)$ such that the buyer utility parameter $a^* \in E(A_t, c_t)$ and $\max\limits_{x\in \mathcal{C}} 2\sqrt{x^TA_tx} \leq \epsilon$.
\end{theorem}

\begin{comment}
Note that the algorithm \algoLUU (in line$:10$) chooses a desired bundle which is used to deduce the separating hyperplane to shrink the uncertainty ellipsoid. The computation $x_t = \text{arg} \max\limits_{x \in \mathcal{C}} \sqrt{x^TA_{t+1}x}$ depends upon the feasibility set $C$. In what follows we briefly discuss the computation of $x_t$.
\end{comment}

\noindent\textbf{Solving the program in line $9$ of \algoLUU:}
\begin{comment}
This is done by equivalently solving the following dual:

The algorithm \algoLUU solves the optimization program $\max\limits_{x \in \mathcal{C}} \sqrt{x^TAx}$, where $A$ is the shape matrix of the uncertainty ellipsoid, or equivalently, of the program 
\begin{equation}
\label{eq:eig_val_prob}
\min\limits_{x \in \mathcal{C}}  \;\; -x^TAx.
\end{equation} 

Since the feasible set $C$ is a quadratic constraint defined by $x^TPx + 2q^Tx + r \leq 0$, where $P \in \mathbf{S}^n, q\in\mathbf{R}^n, r \in \mathbf{R}$, therefore, 
\end{comment}
Even though the program is not convex, but strong duality holds and it can be solved by solving the following dual which is a semidefinite program with variables $\lambda, \gamma \in \mathbf{R}$,
\begin{equation}
\label{eq:eig_val_dual_prob}
\begin{array}{ll@{}ll}
\max  &\gamma&\\
\text{s.t} & \begin{bmatrix}
-A + \lambda P & \lambda q \\
\lambda q^T & \lambda r - \gamma 
\end{bmatrix} \succcurlyeq 0 \\
\end{array}, \\
\lambda \geq 0
\end{equation}
Specifically, if $P = I$ and $q,r=0$, then it becomes an eigenvalue problem.  This leads us to the following corollary of Theorem 
\ref{thm:stackelberg_utility_learn}. 

\begin{corollary}
\label{corr:our corr}	
When the feasible set $C=\{ x \in \mathbf{R}^n : x^Tx \leq 1 \} \subseteq [0,1]^n$, the realistic price set $\mathcal{P}$ and the algorithm parameter $\epsilon$ obey the condition: $R^2\gamma_2 \leq \tau \leq \frac{\epsilon}{4n}$. Then, after at most $t \cdot T$ interactions with the buyer, 
where $t$ amd $T$ are as defined in Theorem \ref{thm:stackelberg_utility_learn}, the algorithm \algoLUU outputs uncertainty set $E(A, c)$ such that the buyer utility parameter $a^* \in E(A, c)$. Moreover, $\lVert a^* - c\rVert_{\infty} \leq \frac{\epsilon}{2}$, i.e, $a^*$ is learned with an accuracy of $\frac{\epsilon}{2}$.  
\end{corollary} 

%$C=\{ x \in \mathbf{R}^n : x^TPx + 2q^Tx + r \leq 0 \} \subseteq [0,1]^n$

\begin{comment}
\begin{proof}
The number of iterations of \algoLUU is given by Theorem \ref{thm:Algo_iter}. Hence, the number of interactions with the buyer follows from Theorem \ref{thm:theorem_for_LearnVal} and \ref{thm:Algo_iter}. Finally, the result follows from Lemma \ref{cohen_eigenvalue_degenerate_lemma}. 
\end{proof}
\end{comment}
\begin{comment}
\begin{lemma}
\label{lemma:volume_cutting}
The algorithm \algoLUU cuts a $ f^{-1} =f(\epsilon, n)^{-1}$ fraction of volume of $E_t$ in each iteration. That is, 
\begin{displaymath}
\frac{\text{VOL}(E_{t+1})}{\text{VOL}(E_t)} \leq f(\epsilon,n)^{-1},
\end{displaymath}
where $f^{-1}$ is $o(1)$. 
\end{lemma}
Assuming there exists a finite volume region, $\mathcal{R}\subseteq \mathbb{R}^n$, such that for any $a \in \mathcal{R}$ and a fixed price vector $p\in \mathbb{R}^n$, the formulation \ref{eq:Buyer's Problem2} have same optimal solution. Then using Lemma \ref{lemma:volume_cutting} and Theorem \ref{thm:theorem_for_LearnVal}, we give the following theorem. 
\begin{theorem}
\label{thm:main_thm_stackelberg}
The algorithm \algoLUU bounds the unknown parameter, $a$, within an ellipsoid $E_r$, such that $E_r \cap \mathcal{R} \neq 0$, in at most  of volume at most $r$ iterations, where $r= \frac{1}{\log f(\epsilon,n)} \log \frac{\text{VOL}(E_1)}{\text{VOL}(\mathcal{R})}$. That is, interacts with the buyer for at most $r\cdot n \cdot poly(\frac{1}{\epsilon}, \frac{1}{\sigma}, \gamma, \lambda_{val})$ rounds.      
\end{theorem}
\end{comment}
\section{Conclusion}
 
In this paper we proposed an efficient online algorithm which can be used by a seller to learn the behavior model of a buyer that maximizes utility, by controlling prices subject to exogenous pricing restrictions. 

One of the key advantages of our algorithm is that it is amenable to exogenous pricing restrictions imposed by business and managerial constraints, making it relatively more practical and user-friendly than previously proposed approaches. Using our algorithm, practitioners can build a model of buyer behavior from purchase and pricing data, which can be subsequently used for inventory, pricing and other business decisions.

\noindent\textbf{Future Work}: We completely side-step the issue of identifiability of the model in our treatment, by reporting uncertainty sets instead of point estimates of the true parameters. When allowable prices are exogenous, it may happen that the best uncertainty set is still very loose due to stringent pricing restrictions. Another important issue that we did not address here is that of modeling stochasticity in the buyer models. 
As our algorithm uses an ellipsoidal search template for  which noisy generalizations exist, it can be extended to the the noisy case (appropriate noise models have to be specified here). Our algorithm also uses projected gradient descent while interacting with the buyer. Thus, noisy gradient information obtained from the buyer can potentially be dealt with as well. 
 %brief discussion

%% The file named.bst is a bibliography style file for BibTeX 0.99c
%\newpage
\bibliographystyle{named}
\bibliography{learn_buyer_ref}

\begin{thebibliography}{}

\bibitem[\protect\citeauthoryear{Alaei}{2011}]{alaei2011bayesian}
Saeed Alaei.
\newblock Bayesian combinatorial auctions: Expanding single buyer mechanisms to
  many buyers.
\newblock In {\em Proceedings of the Fifty-second Annual IEEE Symposium on
  Foundations of Computer Science}, pages 512--521. IEEE Computer Society,
  2011.

\bibitem[\protect\citeauthoryear{Amin \bgroup \em et al.\egroup
  }{2014}]{amin2014repeated}
Kareem Amin, Afshin Rostamizadeh, and Umar Syed.
\newblock Repeated contextual auctions with strategic buyers.
\newblock In {\em Advances in Neural Information Processing Systems}, pages
  622--630, 2014.

\bibitem[\protect\citeauthoryear{Balcan \bgroup \em et al.\egroup
  }{2014}]{Balcan2014}
Maria-Florina Balcan, Amit Daniely, Ruta Mehta, Ruth Urner, and Vijay~V
  Vazirani.
\newblock Learning economic parameters from revealed preferences.
\newblock In {\em International Conference on Web and Internet Economics},
  pages 338--353. Springer, 2014.

\bibitem[\protect\citeauthoryear{Bei \bgroup \em et al.\egroup
  }{2016}]{bei2016learning}
Xiaohui Bei, Wei Chen, Jugal Garg, Martin Hoefer, and Xiaoming Sun.
\newblock Learning market parameters using aggregate demand queries.
\newblock In {\em Proceedings of the Thirtieth AAAI Conference on Artificial
  Intelligence}, 2016.

\bibitem[\protect\citeauthoryear{Beigman and Vohra}{2006}]{beigman2006learning}
Eyal Beigman and Rakesh Vohra.
\newblock Learning from revealed preference.
\newblock In {\em Proceedings of the Seventh ACM Conference on Electronic
  Commerce}, pages 36--42. ACM, 2006.

\bibitem[\protect\citeauthoryear{Besbes and Zeevi}{2015}]{besbes2015surprising}
Omar Besbes and Assaf Zeevi.
\newblock On the (surprising) sufficiency of linear models for dynamic pricing
  with demand learning.
\newblock {\em Management Science}, 61(4):723--739, 2015.

\bibitem[\protect\citeauthoryear{Blum \bgroup \em et al.\egroup
  }{2011}]{blum2011welfare}
Avrim Blum, Anupam Gupta, Yishay Mansour, and Ankit Sharma.
\newblock Welfare and profit maximization with production costs.
\newblock In {\em Proceedings of the Fifty-second Annual IEEE Symposium on
  Foundations of Computer Science}, pages 77--86. IEEE, 2011.

\bibitem[\protect\citeauthoryear{Cai and Daskalakis}{2011}]{cai2011extreme}
Yang Cai and Constantinos Daskalakis.
\newblock Extreme-value theorems for optimal multidimensional pricing.
\newblock In {\em Proceedings of the Fifty-second Annual IEEE Symposium on
  Foundations of Computer Science}, pages 522--531. IEEE, 2011.

\bibitem[\protect\citeauthoryear{Chakraborty \bgroup \em et al.\egroup
  }{2009}]{chakraborty2009dynamic}
T~Chakraborty, Zhiyi Huang, and S~Khanna.
\newblock Dynamic and non-uniform pricing strategies for revenue maximization.
\newblock In {\em Proceedings of the Fiftieth Annual IEEE Symposium on
  Foundations of Computer Science}, 2009.

\bibitem[\protect\citeauthoryear{Cohen \bgroup \em et al.\egroup
  }{2016}]{cohen2016feature}
Maxime Cohen, Ilan Lobel, and Renato~Paes Leme.
\newblock Feature-based dynamic pricing.
\newblock In {\em Proceedings of the ACM Conference on Economics and
  Computation}, 2016.

\bibitem[\protect\citeauthoryear{Gr{\"o}tschel \bgroup \em et al.\egroup
  }{2012}]{grotschel2012geometric}
Martin Gr{\"o}tschel, L{\'a}szl{\'o} Lov{\'a}sz, and Alexander Schrijver.
\newblock {\em Geometric Algorithms and Combinatorial Optimization}, volume~2.
\newblock Springer Science \& Business Media, 2012.

\bibitem[\protect\citeauthoryear{Kleinberg and
  Leighton}{2003}]{kleinberg2003value}
Robert Kleinberg and Tom Leighton.
\newblock The value of knowing a demand curve: Bounds on regret for online
  posted-price auctions.
\newblock In {\em Proceedings of the Fourty-fourth Annual IEEE Symposium on
  Foundations of Computer Science}, pages 594--605. IEEE, 2003.

\bibitem[\protect\citeauthoryear{Letham \bgroup \em et al.\egroup
  }{2016}]{benkdd16}
Benjamin Letham, Lydia~M. Letham, and Cynthia Rudin.
\newblock Bayesian inference of arrival rate and substitution behavior from
  sales transaction data with stockouts.
\newblock In {\em Proceedings of the Twenty-second ACM SIGKDD International
  Conference on Knowledge Discovery and Data Mining}, KDD '16, pages
  1695--1704. ACM, 2016.

\bibitem[\protect\citeauthoryear{Obreshkov}{1963}]{obreshkov1963verteilung}
Nikola Obreshkov.
\newblock {\em Verteilung und Berechnung der Nullstellen Reeller Polynome}.
\newblock Deutscher Verlag der Wissenschaften, 1963.

\bibitem[\protect\citeauthoryear{Roth \bgroup \em et al.\egroup
  }{2016}]{STOC16p949}
Aaron Roth, Jonathan Ullman, and Zhiwei~Steven Wu.
\newblock Watch and learn: Optimizing from revealed preferences feedback.
\newblock In {\em Proceedings of the Forty-eighth ACM Symposium on Theory of
  Computing}, pages 949--962. ACM, 2016.

\bibitem[\protect\citeauthoryear{Sion}{1958}]{sion1958}
Maurice Sion.
\newblock On general minimax theorems.
\newblock {\em Pacific Journal of Mathematics}, 8(1):171--176, 1958.

\end{thebibliography}

\end{document}